\documentclass[letterpaper, 10pt, conference]{ieeeconf}  

\IEEEoverridecommandlockouts                              

\usepackage{mathtools}
\usepackage{amsmath,amssymb,amsfonts}

\usepackage{amsthm}
\newtheorem{theorem}{Theorem}[section]

\newtheorem{assumption}{Assumption}

\newcommand\given[1][]{\:#1\vert\:}

\usepackage{booktabs}
\usepackage{xcolor,subcaption}

\usepackage{hyperref}

\title{\LARGE \bf
Curvature-aware Expected Free Energy as an Acquisition Function for Bayesian Optimization
}

\author{Ajith Anil Meera and Wouter M. Kouw
\thanks{ {This project was supported by the Dutch Research Council (NWO) under grant AiNed XS Europe (number NGF.1609.243.072). Both authors are with the Faculty of Electrical Engineering, TU Eindhoven, The Netherlands,
        {\tt\small a.m.anil1@tue.nl}}}%
}

\begin{document}

\maketitle
\thispagestyle{empty}
\pagestyle{empty}

\begin{abstract}
We propose an Expected Free Energy-based acquisition function for Bayesian optimization to solve the joint learning and optimization problem, i.e., optimize and learn the underlying function simultaneously. We show that, under specific assumptions, Expected Free Energy reduces to Upper Confidence Bound, Lower Confidence Bound, and Expected Information Gain. We prove that Expected Free Energy has unbiased convergence guarantees for concave functions. Using the results from these derivations, we introduce a curvature-aware update law for Expected Free Energy and show its proof of concept using a system identification problem on a Van der Pol oscillator.  On a two-dimensional benchmark with an oscillatory landscape, our adaptive Expected Free Energy acquisition achieves competitive performance in both regret and mean squared error, unlike the typical acquisition functions that perform well in only one metric.
\end{abstract}

\section{INTRODUCTION}
Joint optimization and learning is central to robotics and control, where an agent must acquire an accurate map of the environment (or a phenomenon) and identify high‑value regions (e.g., areas of high human occupancy in search‑and‑rescue or high detection likelihood in target search). Fast and efficient information gathering can improve productivity in precision agriculture, save lives in search and rescue operations, and can aid industrial inspection and maintenance \cite{popovic2024learning}. Since real‑world queries are costly, it is crucial to develop methods that efficiently query at the fastest rate for joint optimization and learning. However, most acquisition functions (AFs) are designed to perform efficient optimization and not to accurately learn the underlying function \cite{frazier2018tutorial}. We aim to bridge this gap by proposing an acquisition function for Bayesian Optimization (BO), that balances exploration and exploitation using the curvature information, not just to optimize, but also to learn the underlying function.
Taking inspiration from active inference \cite{friston2010free}, the leading computational neuroscience theory on how our brain performs inference, we propose Expected Free Energy (EFE) as a new acquisition function. We point towards EFE being a general acquisition function from which other acquisition functions such as Upper Confidence Bound (UCB) \cite{srinivas2012information}, Lower Confidence Bound (LCB) and Expected Information Gain (EIG) can be derived.

BO has an extensive literature on AFs for black-box optimization \cite{garnett2023bayesian}, including information-theoretic regret bounds \cite{srinivas2012information}. Since BO is particularly useful for optimizing multimodal functions that are expensive to evaluate, it has seen widespread use in control systems to tune controllers \cite{khosravi2021performance} and for system identification \cite{chen2025fast}. The idea of using gradient information to improve BO has also been extensively studied, especially with regard to observed gradients \cite{wu2017exploiting,wu2017bayesian} and posterior gradients \cite{ip2026expected}. Our method differs from these approaches as we use the second gradient of the Gaussian process (GP) posterior to adapt the balance between exploration and exploitation. There has been a growing interest in using the ideas from active inference for BO. Recent related work proposed an active inference-based AF by placing a Boltzmann distribution-based energy function on the preference distribution \cite{li2026pragmatic}. 
Our work differs in that we start our EFE derivation specifically from a Gaussian preference prior, as a special case. This route is more elegant because a simple linearization step directly recovers UCB and LCB, without having to make unrealistic assumptions \cite{li2026pragmatic}. We further improve on this EFE by proposing a curvature-aware rule to improve the performance. Another similar work is \cite{kellyboba}, which proposed an active inference-based AF. They used the free energy of the expected future \cite{millidge2021whence} as the objective, instead of EFE, which is fundamentally different from our formulation, and missed out on uncovering the equivalence between UCB and active inference. The core novelty of our work is the use of curvature information within the EFE acquisition function to balance between exploration and exploitation for a joint optimization and learning problem.

We mathematically show that the stationary point of EFE is curvature agnostic. Based on this result, we propose a novel curvature-aware adaptive rule for EFE that encourages resampling in high curvature regions and exploring uncertain regions with low curvature. The core contributions of this paper include: i) a proof that, under specific assumptions, EFE reduces to UCB, LCB and EIG (Sec. \ref{sec:proofs_EIF2AF}), ii) a proof for EFE's unbiased convergence on concave functions (Sec. \ref{sec:EFE_convergence}), iii) introduction of a curvature-aware update rule for EFE (Sec. \ref{sec:adaptiveEFE}), and iv) showing that our adaptive (curvature-aware) EFE  {is competitive with widely-used AFs} for joint optimization and learning (Sec. \ref{sec:experiments}).

\section{PROBLEM STATEMENT}
{We consider BO under a Gaussian process (GP) model for an
unknown latent function $f:\mathcal{X}\to\mathbb{R}$  {with $\mathcal{X}\subseteq\mathbb{R}^d$}}. Conditioned on data $\mathcal{D}$, the GP
posterior at any location $x\in\mathcal{X}$ is
$
f(x)\,\big|\,\mathcal{D} ~\sim~ \mathcal{N}\!\big(\mu(x),\,\sigma^2(x)\big),
$
where $\mu(x)$ and $\sigma^2(x)$ denote the posterior mean and latent posterior
variance, respectively. Observations are corrupted by additive Gaussian noise,
$
y \;=\; f(x) + \varepsilon, 
\  
\varepsilon\sim\mathcal{N}(0,\sigma_n^2),
$
which makes the predictive distribution over noisy observations
\[
q(y\given x) 
= 
\mathcal{N}\!\big(y \given \mu(x),\,\sigma_y^2(x)\big),
\quad  
\sigma_y^2(x) := \sigma^2(x)+\sigma_n^2.
\]
We pose the problem of joint learning and optimization as choosing $x_{t+1} \in \mathcal{X}$ sequentially so that it jointly improves $\mu(x)$ as an estimator of $f$ and maximizes $f(x)$.

\section{EFE AND OTHER ACQUISITION FUNCTIONS}
This section derives EFE for a GP and shows that UCB and LCB are linear surrogates of EFE.

\subsection{EFE under a GP Model}
We view optimization of an unknown function as a problem in which an agent must learn from interactions with its environment to reach a goal. Exploration refers to learning the function (discovering its shape), and exploitation refers to optimization. Action is represented by the choice of the next query location, and the sensor observation is represented by the value of the unknown function at that query location. We implicitly assume that the environment is time-invariant, i.e., $f(x)$ does not change depending on which query points are selected. 
 {Active inference frames perception--action as joint minimization of EFE, which decomposes into a pragmatic term (preference matching) and an epistemic term (information gain) \cite{friston2010free}. This trade-off motivates our choice of EFE as an AF for joint optimization and learning.} We define a one-step ahead (myopic) EFE for a candidate query location $x$ on a GP as \cite{friston2015active}:
\begin{equation}
\label{eq:EFE_def_initial}
\begin{split}
G(x)
~=~
&\underbrace{
\mathbb{E}_{q(y\given x)}
\!\left[-\ln p(y)\right]
}_{\text{pragmatic value}}
~-~
\\&
\underbrace{
\mathbb{E}_{q(y\given x)}
\!\left[
\mathrm{KL}\!\left(
q(f\given y,x)\,\big\|\,q(f\given x)
\right)
\right]
}_{\text{epistemic value}},
\end{split}
\end{equation}
where $\text{KL}(\cdot)$ defines the KL-divergence between two distributions.
The pragmatic term quantifies the expected deviation from the preferred observations,
while the epistemic term measures the expected reduction in uncertainty about the
latent function following an observation at~$x$.
We set a Gaussian preference over the outcomes as
$
p(y)
~=~
\mathcal{N}\!\big(y^\ast,\,\tau^2\big),
$
where $y^\ast$ represents the preferred observation and $\tau^2$ encodes the
variance of this preference. This favors solutions that are closer to $y^*$ in the optimization process.


\begin{assumption} \label{ass:marginal_IG}
We express the GP information gain in terms of the marginal information gain
at the queried location:
\begin{align*}
\mathrm{KL}\!\left(q(f | y,x)\,\|\,q(f | x)\right)
\approx  \mathrm{KL}\!\left(q(f(x) | y,x)\,\|\,q(f(x) | x)\right) ,
\end{align*}
so that the total information gain is captured by the marginal
information gain in the evaluated function value $f(x)$.
\end{assumption}

 {
Assumption~\ref{ass:marginal_IG} expresses the KL in \eqref{eq:EFE_def_initial}
as a scalar quantity at the queried point $x$, rather than as a divergence over
the entire function $f$. Under the pointwise observation model, the information
provided by the observation enters the function through $f(x)$, and the
function-level information gain is exactly characterized by this marginal
quantity. 
}
Under this assumption, the EFE in \eqref{eq:EFE_def_initial} reduces to
\begin{equation}
\label{eq:EFE_def_academic_rephrase}
\begin{split}
G(x)
~=~
&\mathbb{E}_{q(y\given x)}\!\left[-\ln p(y)\right]
~-~
\\
&\mathbb{E}_{q(y\given x)}
\!\left[
\mathrm{KL}\!\big(
q(f(x)\given y,x)\,\big\|\,q(f(x)\given x)
\big)
\right].
\end{split}
\end{equation}
Using $p(y)=\mathcal{N}(y \given y^\ast,\tau^2)$, we have
\[
-\ln p(y)
~=~
\frac{(y-y^\ast)^2}{2\tau^2}
~+~
\frac12 \ln(2\pi \tau^2).
\]
Since $y\sim\mathcal{N}(\mu(x),\sigma_y^2(x))$, we obtain (Appendix~B):
\begin{equation}
\label{eq:pragmatic_term_rephrase}
\mathbb{E}_{q(y\given x)}[-\ln p(y)]
~=~
\frac{(\mu(x)-y^\ast)^2 + \sigma_y^2(x)}{2\tau^2} + \text{const.}
\end{equation}
Under a local GP update with Gaussian likelihood, the expected KL divergence
between posterior and prior at $x$ admits the closed form (Appendix~A)
\begin{equation}
\begin{split}
\label{eq:IG_rephrase}
\mathbb{E}_{q(y\given x)}
\!\left[
\mathrm{KL}\!\big(
q(f(x)\given y,x)\,\|\,q(f(x)\given x)
\big)
\right] \\
~=~
\frac12\ln\!\left(1+\frac{\sigma^2(x)}{\sigma_n^2}\right).
\end{split}
\end{equation}
Substituting \eqref{eq:pragmatic_term_rephrase} and
\eqref{eq:IG_rephrase} into \eqref{eq:EFE_def_academic_rephrase} and
dropping constants yields the compact expression
\begin{equation}
\label{eq:EFE_compact_rephrase}
G(x)
= \underbrace{
\frac{(\mu(x)-y^\ast)^2}{2\tau^2}
~+~
\frac{\sigma_y^2(x)}{2\tau^2} }_{\text{pragmatic value}}
~-~
\underbrace{
\frac12\ln\!\left(1+\frac{\sigma^2(x)}{\sigma_n^2}\right) 
}_{\text{epistemic value}},
\end{equation}
where $\sigma_y^2(x)=\sigma^2(x)+\sigma_n^2$.
The first two terms penalize deviations from preferred outcomes and large
    predictive uncertainty, while the final term rewards informative measurements
through the reduction of latent uncertainty. It can be seen that $\tau^{-2}$ acts as a term that balances exploration and exploitation. When $\tau^{-2}$ is high, the pragmatic value (preference terms) dominates $G(x)$ and contributes to aggressive exploitation. When $\tau^{-2}$ is low, the epistemic value dominates and EFE acts like a pure exploration strategy.  Therefore, our EFE formulation naturally gives us a way to balance exploration and exploitation through $\tau^{-2}$.
 {Although ~\eqref{eq:EFE_def_initial}--\eqref{eq:EFE_compact_rephrase} hold for any input dimension $d$,
the following proofs will be stated for $d=1$ for clarity.}

\subsection{Relations to known acquisitions} \label{sec:proofs_EIF2AF}

EFE is a generalization of the BO acquisition functions as it contains several known AFs. Below, we show the proof for this claim and show the specific restrictions on EFE that lead to LCB, UCB and EIG.

\begin{theorem}[LCB as EFE's linearized pragmatic value]
Consider
\begin{equation}
\label{eq:G_LCB_thm_short}
G(x)
=
\frac{(\mu(x)-y^\ast)^{2}}{2\tau^{2}}
+
\frac{\sigma^{2}(x)}{2\tau^{2}},
\quad
\sigma_y^{2}(x)=\sigma^{2}(x)+\sigma_n^{2},
\end{equation}
without epistemic value.
For a reference point $(\mu_0,\sigma_0)$ with $\sigma_0>0$, the first–order Taylor linearization around $(\mu_0,\sigma_0)$ yields a local linear acquisition
$G(x)\approx a\,\mu(x)+b\,\sigma(x)$ with coefficients $a,b$ determined at $(\mu_0,\sigma_0)$.
If $y^\ast\gg \mu_0$ (so $a<0$), then
\begin{equation}
\label{eq:LCB_final_short_concise}
\arg\min_x G(x)
\equiv
\arg\max_x\, \left[\, \mu(x)-\beta\,\sigma(x)\, \right] \, , 
\end{equation}
for $\beta = b/|a| > 0$, 
i.e., the LCB acquisition.
\end{theorem}

\begin{proof}
Define
$
J(\mu,\sigma)
=
(\mu-y^\ast)^{2} / (2\tau^{2})
+
(\sigma^{2}+\sigma_n^{2})/ (2\tau^{2}).
$
A first-order Taylor expansion at $(\mu_0,\sigma_0)$, after dropping constants, gives:
\[
J(\mu,\sigma)\approx a \mu+b\sigma, \ \
a=\frac{\partial J}{\partial\mu}\Big|_{(\mu_0,\sigma_0)}, \ \
b=\frac{\partial J}{\partial\sigma}\Big|_{(\mu_0,\sigma_0)}.
\]
The partial derivatives are
\[
\frac{\partial J}{\partial\mu}
=
\frac{\mu-y^\ast}{\tau^{2}},
\qquad
\frac{\partial J}{\partial\sigma}
=
\frac{\sigma}{\tau^{2}}.
\]
Evaluating at $(\mu_0,\sigma_0)$ gives
$a=(\mu_0-y^\ast)/\tau^{2}$ and
$b=\sigma_0/\tau^{2}$, i.e., $b>0$.
Hence, the first–order approximation about $(\mu_0,\sigma_0)$ yields
$
G(x)\approx a\,\mu(x)+b\,\sigma(x).
$
If $y^\ast\gg\mu_0$, then $a<0$, and minimizing $a\,\mu(x)+b\,\sigma(x)$ is equivalent to maximizing
\[
\mu(x)-\beta\,\sigma(x),
\qquad
\beta=\frac{b}{|a|}>0,
\]
which establishes \eqref{eq:LCB_final_short_concise}. Therefore, LCB is EFE's linearized pragmatic term. 
\end{proof}

\begin{theorem}[UCB as a Local Linearization of EFE]
Let $J(\mu,\sigma)$ denote the full EFE from \eqref{eq:EFE_compact_rephrase} written in $(\mu,\sigma)$ with $\sigma_y^2=\sigma^2+\sigma_n^2$. 
For a reference $(\mu_0,\sigma_0)$ with $\sigma_0>0$, a first-order Taylor linearization yields a local acquisition of the form $J(\mu,\sigma)\approx a\,\mu+b\,\sigma$. 
Under the regime $\mu_0\ll y^\ast$ and $\sigma_n^2+\sigma_0^2\ll\tau^2$, minimizing $J$ is equivalent to maximizing
\[
\mu(x)+\beta\,\sigma(x),
\qquad
\beta=\frac{-b}{-a}
=\frac{\displaystyle \frac{1}{\sigma_n^2+\sigma_0^2}-\frac{1}{\tau^2}}{\displaystyle \frac{y^\ast-\mu_0}{\tau^2}}
>0,
\]
i.e., a local UCB acquisition.
\end{theorem}

\begin{proof}
By definition, $J(\mu,\sigma)$ is the full EFE from \eqref{eq:EFE_compact_rephrase} expressed in $(\mu,\sigma)$ with $\sigma_y^2=\sigma^2+\sigma_n^2$. 
A first-order Taylor expansion at $(\mu_0,\sigma_0)$ gives
\[
J(\mu,\sigma)\approx a\mu + b\sigma, \ \
a=\frac{\partial J}{\partial\mu}\Big|_{(\mu_0,\sigma_0)}, \ \
b=\frac{\partial J}{\partial\sigma}\Big|_{(\mu_0,\sigma_0)}.
\]
Differentiating \eqref{eq:EFE_compact_rephrase} and using $\sigma_y^2=\sigma^2+\sigma_n^2$ yields
\[
\frac{\partial J}{\partial\mu}=\frac{\mu-y^\ast}{\tau^2},
\qquad
\frac{\partial J}{\partial\sigma}=\frac{\sigma}{\tau^2}-\frac{\sigma}{\sigma_n^2+\sigma^2} \, .
\]
Hence, we have
\[
a=\frac{\mu_0-y^\ast}{\tau^2},
\qquad
b=\sigma_0\!\left(\frac{1}{\tau^2}-\frac{1}{\sigma_n^2+\sigma_0^2}\right).
\]
Minimizing $J$ is equivalent to maximizing $-J$, i.e., $-J(\mu,\sigma)\approx (-a)\,\mu+(-b)\,\sigma$. 
Under the preference regime $\mu_0\ll y^\ast$ and $\sigma_n^2+\sigma_0^2\ll\tau^2$, we have $a<0$ and $b<0$, so
\[
\arg\min_x J(\mu(x),\sigma(x))
~\equiv~
\arg\max_x\big[\mu(x)+\beta\,\sigma(x)\big],
\]
where 
$$\beta=\frac{-b}{-a}
=\frac{\displaystyle \frac{1}{\sigma_n^2+\sigma_0^2}-\frac{1}{\tau^2}}{\displaystyle \frac{y^\ast-\mu_0}{\tau^2}}
>0,$$
which is a local UCB acquisition. Therefore, UCB is a linear surrogate of full EFE.
\end{proof}

\begin{theorem}[EFE's Epistemic Term Equals EIG]
With $G(x)$ defined from \eqref{eq:EFE_compact_rephrase} after dropping the pragmatic value, and using \eqref{eq:IG_rephrase}, we have
\begin{equation}
\begin{split}
\label{eq:IG_ES_1_short_theorem}
G(x)
&=
\mathbb E_{q(y\given x)}
\!\left[
\mathrm{KL}\!\Big(q(f(x)\given y,x)\,\big\|\,q(f(x)\given x)\Big)
\right]\\
&=
\frac{1}{2}\ln\!\Big(1+\frac{\sigma^2(x)}{\sigma_n^2}\Big),
\end{split}
\end{equation}
which is the same as the EIG objective $I\big(f(x);y\given x\big).$
\end{theorem}

\begin{proof}
For the scalar Gaussian pair at fixed $x$,
\[
I\big(f(x);y\given x\big)
=
H\big(f(x)\given x\big)
-
H\big(f(x)\given y,x\big),
\]
with Gaussian entropies
\[
H\big(f(x)\given x\big)=\frac12\ln(2\pi e\,\sigma^2(x)), 
\]
\[
H\big(f(x) | y,x\big) \! = \! \frac12\ln\!\Big(2\pi e\,\frac{\sigma^2(x)\sigma_n^2}{\sigma^2(x)+\sigma_n^2}\Big).
\]
Substitution yields
\begin{align*}
I\big(f(x);y\given x\big)
&=
\frac12\ln\!\Big(1+\frac{\sigma^2(x)}{\sigma_n^2}\Big) \, ,
\end{align*}
which coincides with \eqref{eq:IG_ES_1_short_theorem}. 
\end{proof}

 {Under a simple linearization, EFE reduces to UCB, LCB and EIG, exhibiting itself as a general AF that contains the others. These reductions are meaningful ablations: i) LCB is the pragmatic-only limit of EFE, ii) EIG the epistemic-only limit, and iii) UCB the full linearized EFE. LCB and UCB are thus two operating regimes of the same objective balanced by the parameter $\tau^2$. }  


\section{EFE AS AN ACQUISITION FUNCTION}
In this section, we provide the sufficient condition for which EFE converges in locally strictly concave functions, and we propose a curvature-aware update rule.
 
\subsection{Convergence analysis for EFE-based AF} \label{sec:EFE_convergence}

\begin{theorem}[Sufficient Condition for Unbiased Local Convergence of EFE] \label{thm:EFEconvergence}
Let $f$ have a unique maximizer $x^\star$ and admit
$$f(x)=f(x^\star)-\frac{m}{2}(x-x^\star)^2+O(|x-x^\star|^3)$$ with $m>0$.
Consider the maximization form of EFE:
\[
a_{\mathrm{EFE}}(x) \! = \!
-\Big[
\frac{(\mu(x) \! - \! y^\star)^2}{2\tau^2}
+\frac{\sigma^2(x) \!+ \! \sigma_n^2}{2\tau^2}
-\frac{1}{2}\log(1 +  \frac{\sigma^2(x)}{\sigma_n^2})
\Big]
\]
with $y^\star=\mu(x^\star)$. Then $\tau^2=\sigma^2(x^\star)+\sigma_n^2$ is sufficient for unbiased local convergence:
in a sufficiently small neighborhood of $x^\star$, the maximizer of $a_{\mathrm{EFE}}$ equals $x^\star$.
\end{theorem}

\begin{proof}
Let $h:=x-x^\star$ and assume second--order expansions
\[
\mu(x)=\mu(x^\star)-\frac{m}{2}h^2+O(h^3),
\]
\[
\sigma^2(x)=v_0+g\,h+\frac12 v_2 h^2+O(h^3),
\]
with $v_0=\sigma^2(x^\star)>0$, $g=\sigma^{2\,\prime}(x^\star)$, $v_2=\sigma^{2\,\prime\prime}(x^\star)$.
Set $S:=v_0+\sigma_n^2$ and $\Delta:= 1/(\tau^2) - 1/S$.

The first two terms of $a_{\mathrm{EFE}}(x)$ are evaluated as:
\begin{align*}
\frac{(\mu(x)-y^\star)^2}{2\tau^2} &= \frac{m^2}{8\tau^2}h^4+O(h^5) , \\
\frac{\sigma^2(x)+\sigma_n^2}{2\tau^2} &=\frac{S}{2\tau^2}+\frac{g}{2\tau^2}h+\frac{v_2}{4\tau^2}h^2+O(h^3) .
\end{align*}
For the log term, write
\[1+\frac{\sigma^2(x)}{\sigma_n^2}=\frac{S}{\sigma_n^2}\big(1+\frac{g}{S}h+\frac{v_2}{2S}h^2+O(h^3)\big)\] and use
$\log(1+u)=u - u^2/2+O(u^3)$ to obtain
\[
-\frac12\log\big(1+\frac{\sigma^2(x)}{\sigma_n^2}\big)
=C_0-\frac{g}{2S}h-\Big(\frac{v_2}{4S}-\frac{g^2}{4S^2}\Big)h^2+O(h^3),
\]
with constant $C_0$ independent of $h$. Within the neighborhood of $x^*$,  higher order terms of $h$ are comparatively small and can be neglected. Collecting terms up to $O(h^2)$ gives the quadratic model:
\begin{equation}\label{eq:quad}
a_{\mathrm{EFE}}(h)=C+\tilde L\,h+\tilde Q\,h^2,\ 
\end{equation}
where 
\begin{equation*}
\tilde L= -\frac{g}{2}\Delta \, , \qquad 
\tilde Q= -\frac{v_2}{4}\Delta-\frac{g^2}{4S^2} \, .
\end{equation*}
Define $x_{\mathrm{EFE}}:= \arg\max_{x} a_{\mathrm{EFE}}(x)$ (locally) and
$h_{\mathrm{EFE}}:=x_{\mathrm{EFE}}-x^\star$. Next, we evaluate the stationary point $h_{\mathrm{EFE}}$ of $a_{\mathrm{EFE}}$ by equating its gradient to 0, \[a'_{\mathrm{EFE}}(h)=\tilde L+2\tilde Q h =0,\] 
and solving for $h$, which gives
\begin{equation} \label{eqn:EFE_contraction}
h_{\mathrm{EFE}} = x_{\mathrm{EFE}}-x^\star =-\frac{\tilde L}{2\tilde Q}
=\frac{g\left(\frac{1}{\tau^2}-\frac{1}{S}\right)}
       {v_2\left(\frac{1}{\tau^2}-\frac{1}{S}\right)+\frac{g^2}{S^2}}.
\end{equation}
{Unbiasedness} requires $h_{\mathrm{EFE}}=0$. From (\ref{eqn:EFE_contraction}), this holds if
\[
\Delta=0
\quad\Longleftrightarrow\quad
\frac{1}{\tau^2}=\frac{1}{S}
\quad\Longleftrightarrow\quad
\tau^2=S=\sigma^2(x^\star)+\sigma_n^2.
\]
Local maximality at $h_{\mathrm{EFE}}$ requires $a_{\mathrm{EFE}}''(h_{\mathrm{EFE}})<0$.
Since \eqref{eq:quad} is quadratic, its second derivative is
\[
a_{\mathrm{EFE}}''(h)=2\tilde Q
=-\frac{v_2}{2}\Delta-\frac{g^2}{2S^2} \, .
\]
Hence, at the unbiased setting $\Delta=0$,
\[
a_{\mathrm{EFE}}''(h_{\mathrm{EFE}})=2\tilde Q\big|_{\Delta=0}
=-\frac{g^2}{2S^2}\le 0,
\]
which is strictly negative whenever $g\neq 0$.
Therefore, within the quadratic model, the condition
\[
{\ \tau^2=S=\sigma^2(x^\star)+\sigma_n^2\ }
\]
is \emph{sufficient} for \emph{unbiased local convergence}, i.e., $h_{\mathrm{EFE}}=0$, with a local maximum at $x^\star$. This is strict when $g\neq 0$.
\end{proof}

Two key observations can be made from (\ref{eqn:EFE_contraction}): i) the stationary point of EFE ($h_{\mathrm{EFE}}$) is independent of the curvature $m$ of the underlying function , and ii) the correct choice of $\tau^2$ for EFE to jump into the optimum has an unknown term $\sigma^2(x^\star)$ in it. This makes the EFE's stationary point curvature agnostic. The next section takes inspiration from these observations to define a novel curvature-aware update rule for $\tau^2$.

\subsection{Curvature-aware EFE via adaptive $\tau^2$} \label{sec:adaptiveEFE}
Resampling around high curvature regions provides two advantages, one for optimization and the other for learning: i) discovering and refining a potential optimal solution and ii) a better GP learning around all local optima.  {By Taylor's theorem, the magnitude of $\nabla^2\mu$ governs the leading second-order term of the GP mean around any stationary point. So, refining samples in high-$|\nabla^2\mu|$ regions both improves the local GP fit and confirms candidate extrema.} 
However, exploring low curvature regions with  high uncertainty is also vital. 
An efficient strategy for joint optimization and learning of multi-modal functions would be to balance between exploration and exploitation by i) jumping to the local optima of high curvature regions with low uncertainty and ii) exploring low curvature regions with high uncertainty. Taking inspiration from Theorem \ref{thm:EFEconvergence}, we define a novel update rule for $\tau^2$ to balance exploration and exploitation as:
\begin{equation}
\label{eq:tau_adapt}
    \tau_i^{-2} =   {|\nabla^2\mu( x)|}+ \frac{1}{\sigma^{2}( x)},
\end{equation}
 {where $\nabla^2\mu( x) = \mathrm{tr}(H_\mu( x)) = \sum_{j=1}^d \partial^2\mu/\partial x_j^2$ is the Laplacian of the GP posterior mean. For an ARD-RBF GP, $\nabla^2\mu( x)$ admits a closed form (Appendix~C), so the rule is cheap to evaluate at every candidate $ x$.} The rule is followed by a normalization and rescaling of $\tau^2_i$ to an easily tunable constant range $[\tau^2_{\min},\tau^2_{\max}]$ that represents the pure exploitation and pure exploration strategy:
\begin{equation}
    \tau^2 = \tau_{\min}^2 + (\tau_{\max}^2-\tau_{\min}^2)  \frac{\tau^2_i}{\max(\tau^2_i)}.
\end{equation}
This update rule automatically shifts the EFE in (\ref{eq:EFE_compact_rephrase}) between exploration and exploitation by responding to the local curvature and uncertainty of the GP. When the posterior mean is sharply curved and the predictive variance is small, the update makes $\tau^2$ small, pushing EFE toward exploitation and pulling the next query closer to maxima. Conversely, in regions where the model is flat and uncertain, $\tau^2$ becomes large, encouraging exploration by allowing the acquisition to favor points where the potential information gain is higher. In this way, the update law continuously balances the two behaviors, exploiting structure where the model is confident and exploring where it lacks information. 
 {Extrema located in low-curvature regions are not excluded. Once the surrounding posterior variance is small, the pragmatic term dominates and pulls the next query toward the GP-mean optimum regardless of curvature.}
 
\begin{figure*}[!t]
    \centering

    \begin{subfigure}[t]{0.6\textwidth}
        \centering
        \includegraphics[width=\linewidth]{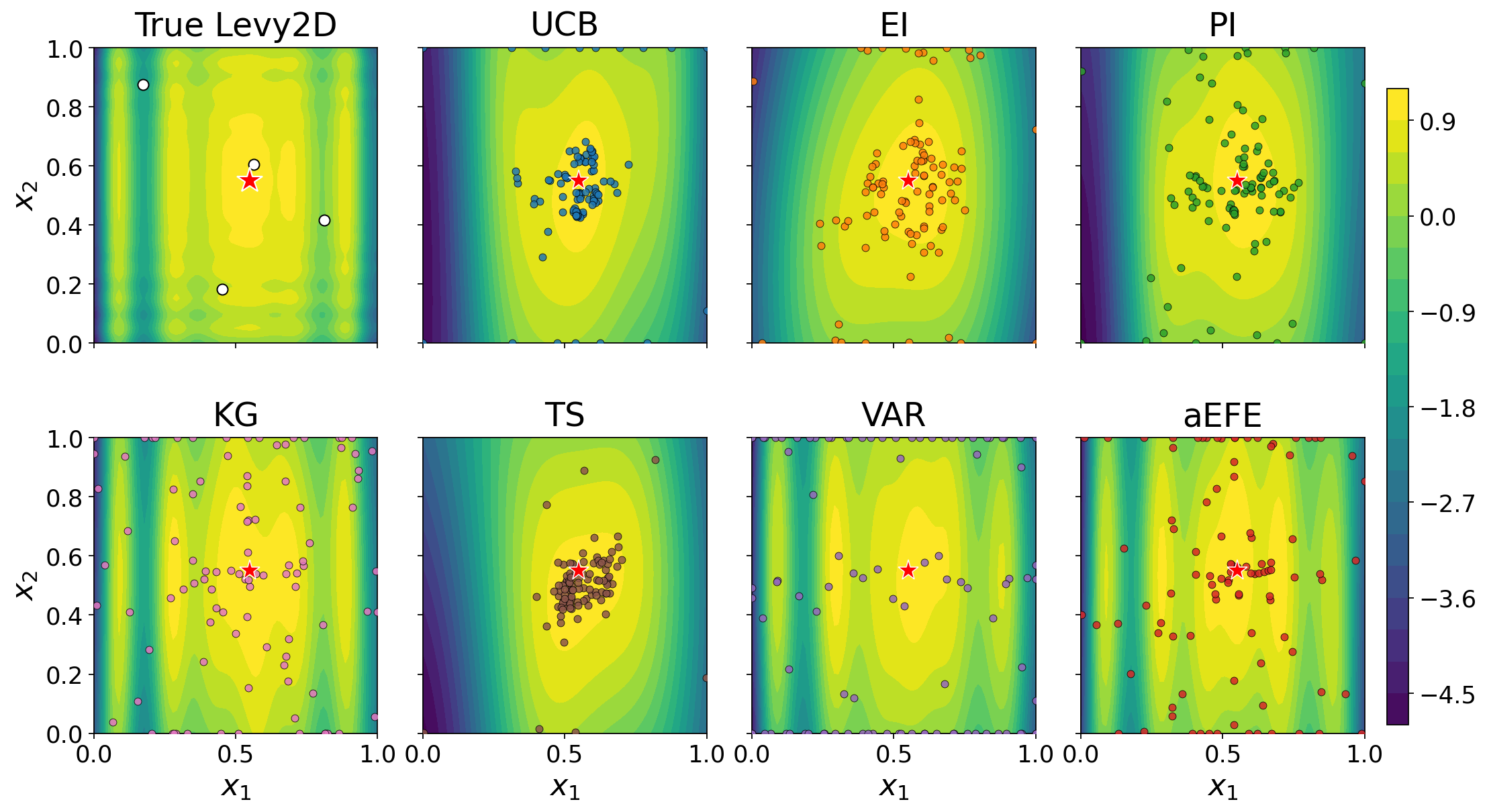}
        \caption{}
        \label{fig:Levy_gp_grid}
    \end{subfigure}
    \hfill
    \begin{subfigure}[t]{0.32\textwidth}
        \centering
        \includegraphics[width=\linewidth]{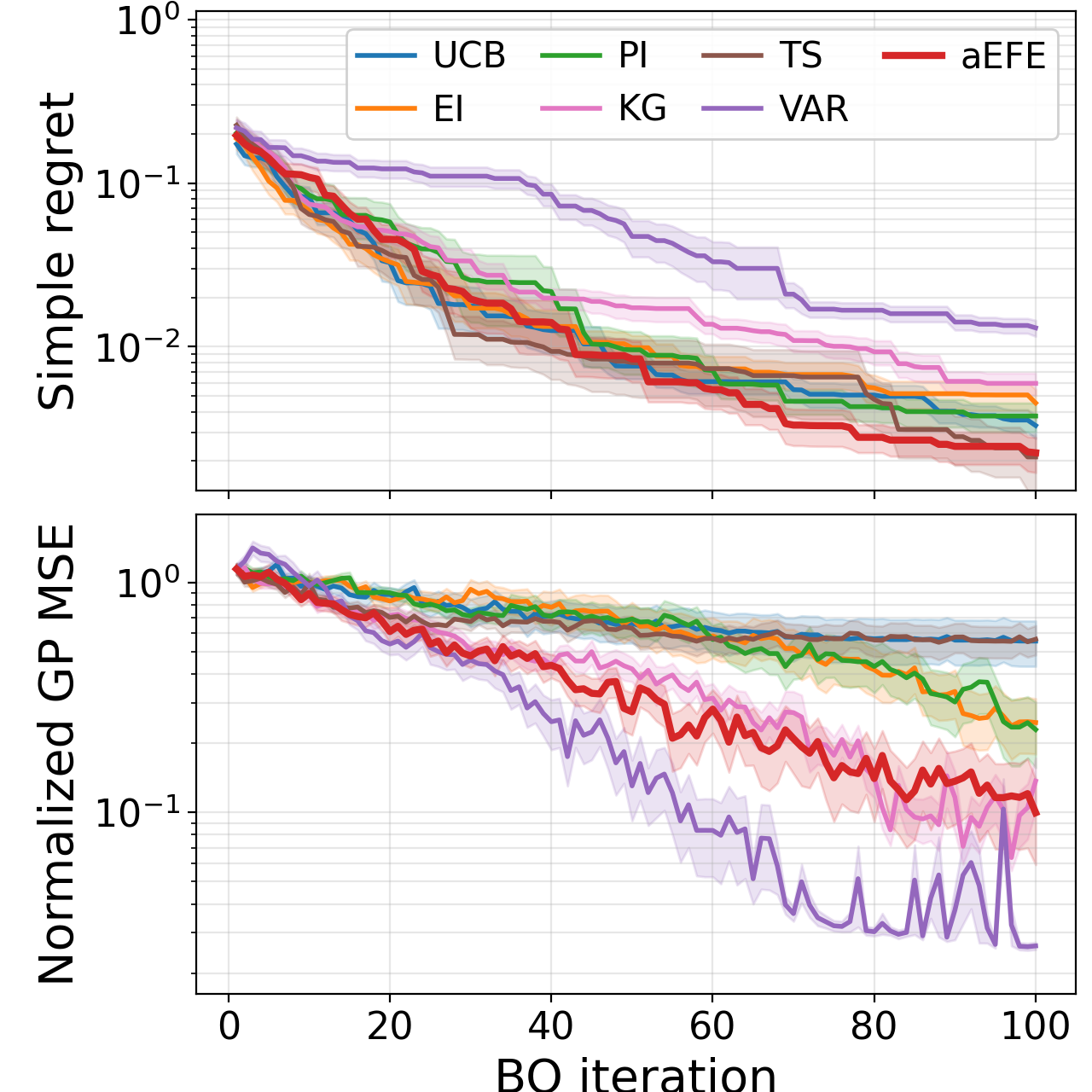}
        \caption{}
        \label{fig:Levy_curves}
    \end{subfigure}

    \caption{ {Results on the 2-D Levy benchmark.
    \textbf{(a)} Standardized negative 2-D Levy and final GP posterior means of each method for a representative seed over the normalized domain $[0,1]^2$. Scatter markers show BO query locations after the Sobol initial design. The red star marks the global maximum.
    \textbf{(b)} Simple regret and normalized GP MSE versus BO iteration. Curves are means over 25 seeds; shaded bands are $\pm 1$ SEM.}}
    \label{fig:Levy_combined}
\end{figure*}

\section{SIMULATION EXPERIMENTS} \label{sec:experiments}
 {This section demonstrates adaptive EFE on a joint learning and optimization problem, first as a proof of concept on a Van der Pol system identification task and then on a two-dimensional Levy benchmark against six other AFs.}

\subsection{Adaptive and non-adaptive EFE}
This section shows the advantage of our adaptive EFE over non-adaptive EFE using a system identification problem on a Van der Pol oscillator given by:
$\ddot{x}-\kappa(1-x^{2})\dot{x}+x=0,$
where $\kappa$ must be inferred. The reference trajectory is generated with the true value $\kappa=3$ from the initial condition $x(0)=0.5$, $\dot{x}(0)=0$ with noise $\sigma_n=0.1$ and sampling time $\Delta t = 0.05s$. the objective is evaluated only on the steady-state window $t\in[20,60]$ to avoid transient artefacts. The BO objective is the negative MSE between the simulated trajectory at $\kappa$ and the noisy reference.  {Both variants share identical initial points, seeds, 50-iteration budgets, GP settings and observation noise; only $\tau^2$ differs (adaptive vs.\ fixed).} Figure~\ref{fig:adaptive_EFE} shows the BO results overlaid on the true objective. Although both variants find the optimum at $\kappa=3$, adaptive EFE explores all high-curvature regions and thus learns the cost function more accurately, whereas non-adaptive EFE skips low-valued high-curvature regions.

\begin{figure}[!htb]
    \centering
    \includegraphics[width=.7\linewidth]{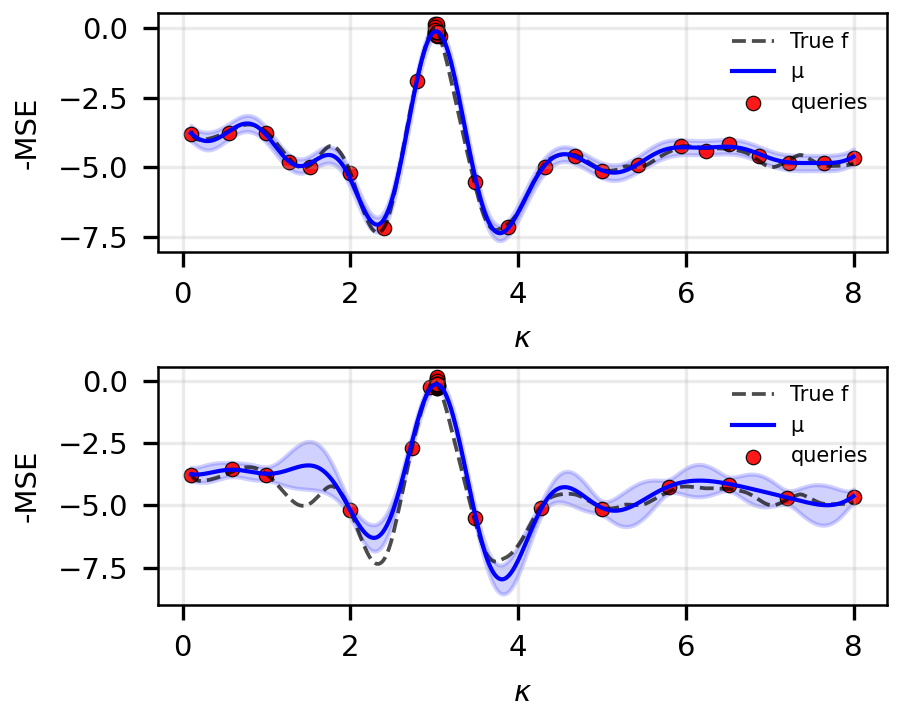}
    \caption{ {BO trajectories of adaptive (top) and non-adaptive (bottom) EFE on the Van der Pol parameter-estimation objective. Markers indicate sampled $\kappa$ values; the solid curve is the true negative-MSE objective.}}
    \label{fig:adaptive_EFE}
\end{figure}

\subsection{Learning and Optimization }

 {We benchmark adaptive EFE (aEFE) against six widely-used AFs on the 2D Levy benchmark: UCB with a time-varying schedule~\cite{srinivas2012information}, Expected Improvement (EI)~\cite{frazier2018tutorial}, Probability of Improvement (PI)~\cite{shahriari2015taking}, Knowledge Gradient (KG)~\cite{frazier2018tutorial}, Thompson Sampling (TS)~\cite{shahriari2015taking}, and pure variance maximization (VAR). The Levy benchmark is a multi-modal oscillatory landscape with a single global maximum surrounded by sinusoidal ridges. Inputs are normalized to $[0,1]^2$, and the global maximum appears near the centre of Fig.~\ref{fig:Levy_gp_grid}. The global maximum is not the highest-curvature point of the domain. So, curvature-aware AFs are not artificially favored. Results are averaged over 25 random seeds. The full experimental setup is given in Appendix C.}


 {Fig.~\ref{fig:Levy_curves} reports simple regret and normalized GP MSE versus BO iteration for 25 seeds, and Fig.~\ref{fig:Levy_gp_grid} shows each method's final GP posterior (for one realization) alongside the true Levy landscape. Only KG, VAR and aEFE recover the true underlying function (top-left panel of Fig.~\ref{fig:Levy_gp_grid}), while UCB, EI, PI and TS over-concentrate near the optimum and miss the oscillatory ridges. Among the function-recovering methods, VAR and KG over-explores and pays for it in regret (Fig.~\ref{fig:Levy_curves}), whereas aEFE places more queries around the the optimum (red star in Fig.~\ref{fig:Levy_gp_grid}), attaining the lowest final simple regret while continuing to reduce its GP MSE (Fig.~\ref{fig:Levy_curves} ). This demonstrates that EFE is well-suited for joint optimization and learning.}

\section{CONCLUSIONS}
EFE has a strong potential as a next‑generation BO acquisition as it unifies exploration and exploitation in a principled way. Taking inspiration from this, we introduced a novel curvature-aware EFE acquisition function for BO for the joint learning and optimization problem. We mathematically showed that UCB and LCB are linear surrogates of EFE, and that EFE has unbiased convergence guarantees for concave functions. Our simulation results show that EFE is very competitive in comparison with  {widely-used AFs}. Future work can develop scalable, MPC‑style multi‑step EFE planners that propagate GP beliefs along candidate action sequences, enabling non‑myopic query selection that anticipates the full evolution of future posterior states. 
 {Beyond this multi-step extension, three further directions are natural. First, the Gaussian preference may be generalized to exponential-family preferences without changing the EFE decomposition. Second, while GPs are well suited to the low-to-moderate-dimensional controller-tuning and system-identification settings considered here~\cite{khosravi2021performance,chen2025fast}, higher-dimensional problems could use sparse or low-rank GP surrogates~\cite{garnett2023bayesian} as direct replacements. Third, the local convergence result in Theorem~\ref{thm:EFEconvergence} could be extended toward global multi-modal convergence using existing UCB-type regret bounds~\cite{srinivas2012information,bogunovic2016time}.}




\section*{APPENDIX}

\subsection{Evaluation of the Epistemic Value}

Let $f(x)\sim \mathcal N(\mu,\sigma^2)$ and $y=f(x)+\varepsilon$, with
$\varepsilon\sim\mathcal N(0,\sigma_n^2)$. Then,
\[
q(y\given x)=\mathcal N(y \given \mu,\sigma^2+\sigma_n^2),\quad
q(f(x)\given y)=\mathcal N(\mu^+,\sigma_+^2),
\]
where
$
\sigma_+^2=\frac{\sigma^2\sigma_n^2}{\sigma^2+\sigma_n^2}.
$
The KL divergence between posterior and prior ($\Phi = 
\mathrm{KL}\!\left(q(f(x)\given y)\,\|\,q(f(x))\right)$):
\[ \Phi
=
\frac12\left[
\frac{\sigma_+^2}{\sigma^2}
+
\frac{(\mu^+-\mu)^2}{\sigma^2}
-1
+
\ln\frac{\sigma^2}{\sigma_+^2}
\right].
\]

Taking expectation over $q(y\given x)$,
\[
\mathbb E_{q(y\given x)}
\!\left[
\Phi
\right]
=
\frac12\left[
\frac{\sigma_+^2}{\sigma^2}
+
\frac{\mathbb E[(\mu^+-\mu)^2]}{\sigma^2}
-1
+
\ln\frac{\sigma^2}{\sigma_+^2}
\right].
\]

Using the Kalman update
$
\mu^+=\mu+K(y-\mu),\quad
K=\frac{\sigma^2}{\sigma^2+\sigma_n^2},
$
we obtain
\[
\mathbb E[(\mu^+-\mu)^2]
=
K^2(\sigma^2+\sigma_n^2)
=
\frac{\sigma^4}{\sigma^2+\sigma_n^2}
=
\sigma^2-\sigma_+^2.
\]

Substituting yields
\[
\mathbb E_{q(y\given x)}
\!\left[
\mathrm{KL}\!\left(q(f(x)\given y)\,\|\,q(f(x))\right)
\right]
=
\frac12\log\left(1+\frac{\sigma^2}{\sigma_n^2}\right).
\]



\subsection{Evaluation of the Pragmatic Term}
Let $q(y\given x)=\mathcal N(y \given \mu,\sigma_y^2)$ and 
$p(y)=\mathcal N(y \given y^\ast,\tau^2)$.
The pragmatic term is the cross-entropy
$\mathbb E_{q(y\given x)}[-\ln p(y)]$.
Since $-\ln p(y) = (y-y^\ast)^2/ (2\tau^2) + \ln(2\pi\tau^2)/2$ and
\begin{align*}
\mathbb E_{q(y\given x)}[(y-y^\ast)^2]
&= \operatorname{Var}_{q(y\given x)}(y)+(\mathbb E_{q(y\given x)}[y]-y^\ast)^2 \\
&= \sigma_y^2+(\mu-y^\ast)^2,
\end{align*}
the cross-entropy is
\[
\mathbb E_{q(y\given x)}[-\ln p(y)]
=
\frac{(\mu-y^\ast)^2+\sigma_y^2}{2\tau^2}
+
\frac12\ln(2\pi\tau^2).
\]

\subsection{Simulation settings} \label{app:simulation}
 {The AFs in Sec.~\ref{sec:experiments}.B are evaluated under identical settings on 2-D Levy using BoTorch. The GP is a \texttt{SingleTaskGP} with an RBF-ARD kernel and an output \texttt{Standardize} transform, with length-scales and signal/noise variance refit online at every BO iteration by marginal-likelihood maximization. The noise is $\sigma_n=0.2$. The initial design has $4$ Sobol points, each run lasts $100$ iterations, and results are averaged over $25$ random seeds. UCB uses $\beta_t = 2\log\bigl(t^2\pi^2/(6\delta)\bigr)$ ~\cite{srinivas2012information}. EI, PI, KG (4 fantasies), TS (2048-Sobol pool argmax) and VAR follow standard BoTorch implementations. For aEFE we use $\tau_{\min}^2=\sigma_n^2 = 0.04$, $\tau_{\max}^2=4.0$, and an adaptive target $y^\ast=y_{\mathrm{best}}+0.5\,\sigma_{\max}$ . Each AF is maximized by multi-start L-BFGS ($5\times 128$ raw samples; KG: $3\times 64$). For an ARD-RBF GP, the Laplacian of $\mu$ has the closed form $\nabla^2\mu(x)=\sum_j \alpha_j\,k(x,x_j)\,[\sum_i (x_i-x_{j,i})^2/l_i^4-\sum_i 1/l_i^2]$ with $\alpha=(K+\sigma_n^2 I)^{-1}(y-m)$; $\alpha$ is precomputed once per BO iteration. The normalized GP MSE in Fig.~\ref{fig:Levy_curves} is evaluated on a fixed $4096$-Sobol grid as $\mathrm{MSE}_n=\mathbb{E}[(\mu(x)-f(x))^2]/\mathrm{Var}(f)$.} Supporting code: \url{https://github.com/biaslab/LCSS2026-EFEaqBO}



\bibliographystyle{IEEEtran}
\footnotesize
\bibliography{references}

@article{li2026pragmatic,
  title={Pragmatic Curiosity: A Hybrid Learning-Optimization Paradigm via Active Inference},
  author={Li, Yingke and Parashar, Anjali and Zhou, Enlu and Fan, Chuchu},
  journal={arXiv:2602.06104},
  year={2026}
}

@article{friston2010free,
  title={The free-energy principle: a unified brain theory?},
  author={Friston, Karl},
  journal={Nature Reviews Neuroscience},
  volume={11},
  number={2},
  pages={127--138},
  year={2010},
  publisher={Nature publishing group}
}

@article{wu2017exploiting,
  title={Exploiting gradients and {Hessians} in {Bayesian} optimization and {Bayesian} quadrature},
  author={Wu, Anqi and Aoi, Mikio C and Pillow, Jonathan W},
  journal={arXiv:1704.00060},
  year={2017}
}

@article{wu2017bayesian,
  title={Bayesian optimization with gradients},
  author={Wu, Jian and Poloczek, Matthias and Wilson, Andrew G and Frazier, Peter},
  journal={Advances in Neural Information Processing Systems},
  volume={30},
  year={2017}
}

@article{frazier2018tutorial,
  title={A tutorial on {Bayesian} optimization},
  author={Frazier, Peter I},
  journal={arXiv:1807.02811},
  year={2018}
}

@article{srinivas2012information,
  title={Information-theoretic regret bounds for {Gaussian} process optimization in the bandit setting},
  author={Srinivas, Niranjan and Krause, Andreas and Kakade, Sham M and Seeger, Matthias W},
  journal={IEEE Transactions on Information Theory},
  volume={58},
  number={5},
  pages={3250--3265},
  year={2012},
  publisher={IEEE}
}

@article{ip2026expected,
  title={Expected Improvement via Gradient Norms},
  author={Ip, Joshua Hang Sai and Makrygiorgos, Georgios and Mesbah, Ali},
  journal={arXiv:2601.21357},
  year={2026}
}

@article{kellyboba,
  title={{BOBA}: Dynamic {Bayesian} Optimization through {Bayesian} Active Inference},
  author={Kelly, Merlin and Patel, Rishan and Thomas, Alexander and Zhu, Ziyue and Quan, Zikun and Carlson, Tom and Cho, Youngjun}
}

@article{millidge2021whence,
  title={Whence the expected free energy?},
  author={Millidge, Beren and Tschantz, Alexander and Buckley, Christopher L},
  journal={Neural Computation},
  volume={33},
  number={2},
  pages={447--482},
  year={2021},
  publisher={MIT Press}
}

@article{shahriari2015taking,
  title={Taking the human out of the loop: A review of {Bayesian} optimization},
  author={Shahriari, Bobak and Swersky, Kevin and Wang, Ziyu and Adams, Ryan P and De Freitas, Nando},
  journal={Proceedings of the IEEE},
  volume={104},
  number={1},
  pages={148--175},
  year={2015},
  publisher={IEEE}
}

@article{khosravi2021performance,
  title={Performance-driven cascade controller tuning with {Bayesian} optimization},
  author={Khosravi, Mohammad and Behrunani, Varsha N and Myszkorowski, Piotr and Smith, Roy S and Rupenyan, Alisa and Lygeros, John},
  journal={IEEE Transactions on Industrial Electronics},
  volume={69},
  number={1},
  pages={1032--1042},
  year={2021},
  publisher={IEEE}
}

@article{chen2025fast,
  title={Fast Kernel-Based Regularized System Identification using {Bayesian} Optimization},
  author={Chen, Lujing and Chen, Tianshi and Andersen, Martin S},
  journal={IEEE Transactions on Automatic Control},
  year={2025},
  publisher={IEEE}
}

@article{popovic2024learning,
  title={Learning-based methods for adaptive informative path planning},
  author={Popovi{\'c}, Marija and Ott, Joshua and R{\"u}ckin, Julius and Kochenderfer, Mykel J},
  journal={Robotics and Autonomous Systems},
  volume={179},
  pages={104727},
  year={2024},
  publisher={Elsevier}
}

@article{friston2015active,
  title={Active inference and epistemic value},
  author={Friston, Karl and Rigoli, Francesco and Ognibene, Dimitri and Mathys, Christoph and Fitzgerald, Thomas and Pezzulo, Giovanni},
  journal={Cognitive Neuroscience},
  volume={6},
  number={4},
  pages={187--214},
  year={2015},
  publisher={Taylor \& Francis}
}

@book{garnett2023bayesian,
  title={Bayesian Optimization},
  author={Garnett, Roman},
  year={2023},
  publisher={Cambridge University Press}
}

@inproceedings{bogunovic2016time,
  title={Time-varying {Gaussian} process bandit optimization},
  author={Bogunovic, Ilija and Scarlett, Jonathan and Cevher, Volkan},
  booktitle={Artificial Intelligence and Statistics},
  pages={314--323},
  year={2016},
  organization={PMLR}
}

\end{document}